\documentclass[conference]{IEEEtran}
\IEEEoverridecommandlockouts

\usepackage{cite}
\usepackage{amsmath,amssymb,amsfonts}
\usepackage{algorithmic}
\usepackage{graphicx}
\usepackage{textcomp}
\usepackage{xcolor}

\usepackage{xcolor}
\usepackage{bigints}
\usepackage{amsthm}

\newtheorem{remark}{\bfseries Remark}

\newtheorem{problem}{\bfseries Problem}
\newtheorem{assumption}{Assumption}

\DeclareMathOperator*{\argmin}{arg\,min}
\DeclareMathOperator*{\KL}{D_{KL}}

\DeclareMathOperator*{\W2}{W_2^2}
\DeclareMathOperator*{\diag}{diag}
\DeclareMathOperator*{\Gaussian}{\mathcal{N}}

\DeclareMathOperator*{\pms}{ {\scriptstyle \pm}}
\usepackage{multirow}
\usepackage{amsmath}
\usepackage{algorithm}
\usepackage{algorithmic}
\usepackage{tikz}
\usetikzlibrary{shapes, arrows, positioning}
\usepackage{subcaption}
\usepackage{booktabs}

\theoremstyle{plain}
\newtheorem{theorem}{Theorem}[section]

\theoremstyle{definition}

\theoremstyle{remark}

\begin{document}
\bstctlcite{IEEEexample:BSTcontrol}

\title{Information-Geometric Barycenters for Bayesian Federated Learning\\}

\author{
    \IEEEauthorblockN{
        Nour Jamoussi\IEEEauthorrefmark{2}, 
        Giuseppe Serra\IEEEauthorrefmark{2}, 
        Photios A. Stavrou\IEEEauthorrefmark{2}, 
        Marios Kountouris\IEEEauthorrefmark{2}\IEEEauthorrefmark{3}
    }
    \IEEEauthorblockA{\IEEEauthorrefmark{2}{Communication Systems Department, EURECOM, France}}
    \IEEEauthorblockA{\IEEEauthorrefmark{3}{Andalusian Institute of Data Science and Computational Intelligence (DaSCI)}\\
    \IEEEauthorrefmark{0}{Department of Computer Science and Artificial Intelligence, University of Granada, Spain}}

}

\maketitle

\begin{abstract}
Federated learning (FL) is a widely used and impactful distributed optimization framework that achieves consensus by averaging locally trained models. While effective, this approach may not align well with Bayesian inference, where the model space is more naturally represented as a distribution space. Taking an information-geometric perspective, we reinterpret FL aggregation as the problem of finding the barycenter of local posteriors using a predefined divergence metric, minimizing the average discrepancy across clients. This perspective provides a unifying framework that generalizes many existing methods and offers crisp insights into their theoretical underpinnings. We then propose BA-BFL, an algorithm that retains the convergence properties of Federated Averaging in non-convex settings. In non-independent and identically distributed scenarios, we conduct extensive comparisons with statistical aggregation techniques, showing that BA-BFL achieves performance comparable to state-of-the-art methods while also providing a geometric interpretation of the aggregation phase. Additionally, we extend our analysis to Hybrid Bayesian Deep Learning, exploring the impact of Bayesian layers on uncertainty quantification and model calibration.
\end{abstract}


\begin{figure*}[!t]
    \centering
    \begin{minipage}{0.48\linewidth}
        \centering
        \includegraphics[width=\linewidth]{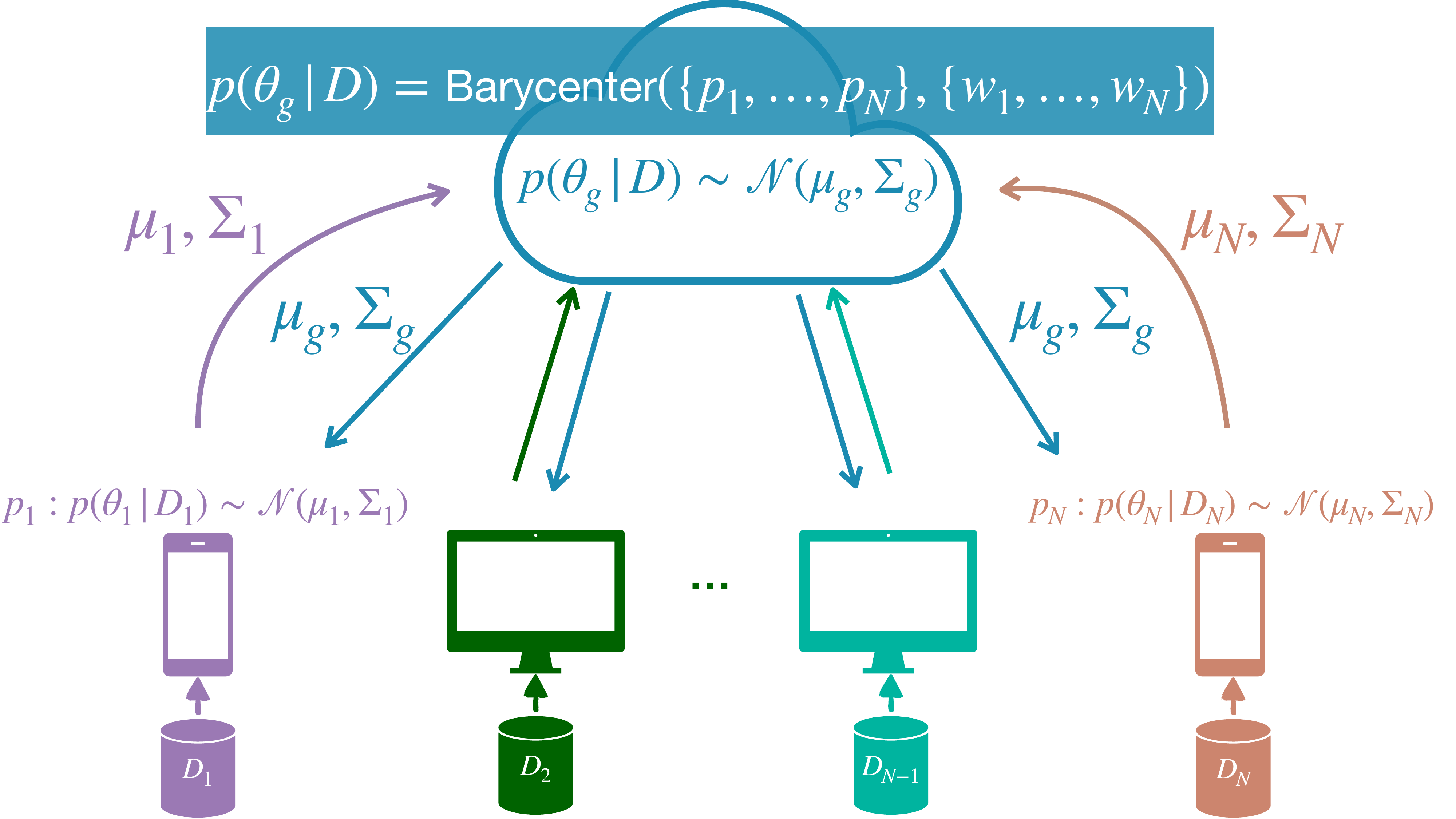}
        \caption{The BA-BFL framework.}
        \label{fig:BFL}
    \end{minipage}
    \hfill
    \begin{minipage}{0.48\linewidth}
        \centering
        \centering
    \resizebox{\columnwidth}{!}{
    \begin{tikzpicture}[
        every node/.style={draw, rounded corners, align=center, font=\small, minimum width=4cm, minimum height=1.5cm},
        node distance=0.3cm and 1.2cm,
        box1/.style={fill=blue!20},
        box2/.style={fill=purple!20},
        arrowstyle/.style={thick, ->, shorten >=2pt, shorten <=2pt}
    ]

    \node[box1] (arithmetic_methods) 
    {\parbox{4cm}{\centering \textbf{Deterministic Arithmetic Average Aggregations}\\ 
    \textit{FedAvg \cite{mcmahan2017communication}, \\FedProx \cite{li2022federated}, PropFair \cite{zhang2022proportional}, \\ GiFair \cite{yue2023gifair},  FedLaw \cite{li2023revisiting} }}};

    \node[box1, below=of arithmetic_methods] (multiplicative_methods) 
    {\parbox{4cm}{\centering \textbf{Multiplicative Aggregations of Posteriors}\\ 
    \textit{FedPA \cite{al2020federated}, FOLA \cite{liu2023bayesian}, \\ FedEP \cite{guo2023federated}, FedIvon \cite{pal2024simple}}}};

    \node[box2, right=of arithmetic_methods] (corresponding1) 
    {\parbox{4cm}{\centering \textbf{WB / RKLB \\ with variance $\to$ 0}}};

    \node[box2, right=of multiplicative_methods] (corresponding2) 
    {\parbox{4cm}{\centering \textbf{RKLB}}};

    \draw[arrowstyle] (arithmetic_methods.east) -- (corresponding1.west);
    \draw[arrowstyle] (multiplicative_methods.east) -- (corresponding2.west);

    \end{tikzpicture}
    }
    \caption{Mapping of various aggregation methods to their corresponding barycenter formulations.}
    \label{fig:aggregation_mapping}
    \end{minipage}

\end{figure*}

\section{Introduction}
Federated Learning (FL) has emerged as the de facto standard for decentralized learning, particularly in scenarios that demand strong privacy guarantees. As originally introduced in \cite{mcmahan2017communication}, an FL system consists of a central server that maintains a global model and interacts with multiple clients (end-user devices), each of which holds private local data. 
FL schemes typically operate in two phases. In the local learning phase, each client trains a model on its own private data. In the aggregation phase, the locally updated models are transmitted to the server and merged according to a predetermined rule. These phases repeat iteratively, with the global model from the previous iteration distributed to clients as the starting point for their local models in the current iteration.

Aggregation plays a central role in FL, allowing individual client contributions to be combined into a global model while preserving data privacy and ensuring communication efficiency.
Although various aggregation strategies have been proposed (e.g., \cite{qi2023model}), most aggregation methods rely on variants of weighted averaging. Notable examples include FedAvg \cite{mcmahan2017communication} and FedProx \cite{li2022federated}, which aggregate local models by computing a weighted average of their parameters. The choice of weights is an important design parameter, as it can encode auxiliary attributes such as the relative importance of each client’s model to the overall objective or reflect the amount of data available to each client.

A key challenge in FL, and generally in distributed learning systems, is the statistical heterogeneity among participating clients. In real-world scenarios, client datasets rarely satisfy the idealized assumption of independent and identically distributed (i.i.d.) data. Instead, they often exhibit significant heterogeneity and distributional shifts across clients. 
As reviewed in \cite{kairouz2021advances,li2022federated}, five major forms of heterogeneity are typically identified: \textit{label distribution skew}: differences in the frequency of specific labels across clients (e.g., under- or over-represented classes); \textit{feature distribution skew}: differences in the feature distributions associated with the same label; \textit{concept drift}: cases where the same label is associated with different feature distributions across clients; \textit{concept shift}: cases where identical samples receive different labels from different clients, and \textit{quantity skew}: differences in the number of data samples held by each client.
While these heterogeneities are of high practical relevance, addressing all of them simultaneously remains a significant challenge. As a result, most of the existing FL research is focused on only a subset of these challenges \cite{deng2020adaptive,t2020personalized,fallah2020personalized}.
Given the growing applicability of FL in real-world settings, uncertainty quantification and model calibration are central to building trustworthy and reliable models. Nonetheless, these aspects remain largely underexplored in existing research on deterministic FL. A preliminary study on the topic, with a specific application to healthcare, is presented in \cite{zhang2023uncertainty}. It provides an overview of various uncertainty quantification methods for deterministic FL, which were later implemented in \cite{koutsoubis2024privacy}. However, it is important to note that the techniques discussed are primarily inspired by Bayesian approaches, such as Bayesian ensembles and Monte Carlo dropout.

Bayesian learning excels at improving model reliability, as Bayesian methods enable more accurate uncertainty quantification and calibration, making them a compelling solution for FL in non-i.i.d. contexts. FedPPD \cite{bhatt2024federated} introduces an FL framework with built-in uncertainty quantification: in each round, each client estimates both the posterior distribution over its parameters and the posterior predictive distribution (PPD). The PPD is subsequently distilled into a single deep neural network, which is then sent to the server.
pFedBayes \cite{zhang2022personalized} and Fedpop \cite{kotelevskii2022fedpop} are also Bayesian approaches with a focus on uncertainty quantification aspects, proposed in the context of personalized Bayesian Federated Learning (BFL). Nonetheless, we emphasize that many of the existing methods \cite{zhang2022personalized,ozer2022combine,bhatt2024federated,fischer2024federated} rely on variations of weighted averaging of the posteriors' parameters.

\paragraph{Contributions}
This work introduces a unifying perspective on aggregation methods in BFL through the lens of barycentric aggregation (BA-BFL). Given a divergence metric, we interpret the aggregation process as a geometric problem, where the global model is identified as the barycenter of the local posteriors. Unlike existing methods that often rely on heuristic variations of parameter averaging, our approach is theoretically grounded: it explicitly minimizes the average divergence between the global posterior and the local posteriors. We show that this methodology generalizes several aggregation strategies previously proposed in the literature. Furthermore, the proposed methods preserve the convergence properties of FedAvg, even in non-convex settings (see Theorem \ref{th:convergence}). We evaluate our approach against state-of-the-art Bayesian aggregation methods, comparing accuracy and uncertainty quantification in heterogeneous settings. The results demonstrate that our method achieves performance comparable to existing statistical aggregation techniques. To bridge gaps in the BFL literature and buildinsights from Hybrid Bayesian Deep Learning (HBDL) \cite{zeng2018relevance}, we further examine how limiting the number of Bayesian layers affects the performance of different Bayesian aggregation methods.

\paragraph{Notation}
Table~\ref{tab:notation} summarizes the key symbols and their meanings for clarity and convenience of reference.

\begin{table}[h!]
\centering
\small
\renewcommand{\arraystretch}{0.8}
\begin{tabular}{ll}
\toprule
\textbf{Symbol} & \textbf{Meaning} \\
\midrule
\( \mathcal{D}_k \) & Local dataset of client \(k\) \\
\( \mathcal{D} \) & Union of all datasets \( \bigcup_{k=1}^N \mathcal{D}_k \) \\
\( D_\alpha \) & \(\alpha\)-divergence \\
\( \KL \) & Kullback–Leibler divergence \\
\( \W2 \) & Squared 2-Wasserstein distance \\
\( w_k \) & Weight of client \(k\) \\
\( f_k \) & Local objective of client \(k\) \\
\( f \) & Weighted sum \( \sum_{k=1}^N w_k f_k \) \\
\( \psi_k \) & Posterior parameters of client \(k\)'s model \\
\( \psi_g \) & Posterior parameters of the global model \\
\bottomrule
\end{tabular}
\caption{Summary of notation.}
\label{tab:notation}
\end{table}

\section{Background and Related Work}
\paragraph{Federated Learning}
An FL system \cite{mcmahan2017communication} consists of a central server and $N$ clients that engage in an iterative learning process through server-client communication. For each communication round, the $k^{th}$ client trains its local model, parameterized by $\theta_k$, on its private data $\mathcal{D}_k$. Subsequently, the model parameters $\theta_k$ are sent to the server, which aggregates them to obtain the global model. The updated global model is then distributed back to the clients to refine their local models in the next communication round. Through this process, FL aims to learn a global model $\theta^*$ on the aggregated dataset $\mathcal{D} = \bigcup_{k = 1}^N \mathcal{D}_k$ from all participating clients. 

In general, the objective function in FL takes the form
\begin{align}
    \min _{{\theta}} f({\theta}) = \sum_{k=1}^N w_k f_k({\theta}) 
\label{eq:fl_objective}
\end{align}
where $f_k({\theta})= \mathbb{E}_{(x,y) \sim \mathcal{D}_k}[\mathcal{L}(\theta ; (x,y))]$ is the local objective function of the $k^{th}$ client, and $w_k$ is its associated weight, with $\sum_{k = 1}^N w_k = 1$. At each communication round, minimizing $f_k(\theta)$ locally produces the client update $\theta_k$.

\begin{table*}[ht]
    \centering
    \renewcommand{\arraystretch}{1}

    \begin{tabular}{@{}lll@{}}
        \toprule
        \textbf{BFL Method} & \textbf{Local Bayesian Technique} & \textbf{Global Aggregation Method} \\ \midrule
        FedPA \cite{al2020federated} & MCMC sampling & Multiplicative aggregation of posteriors \\
        pFedBayes \cite{zhang2022personalized} & Variational Inference  & Weighted averaging of posteriors' statistics\\
        FVBA \cite{ozer2022combine} & Variational Inference & Weighted averaging of posteriors' statistics \\
        FedEP \cite{guo2023federated} & Variational Inference & Multiplicative aggregation of posteriors \\
        FedHB \cite{kim2023fedhb} & Variational Inference & Bayesian posterior inference  \\
        FOLA \cite{liu2023bayesian}  & Laplace Approximation  & Multiplicative aggregation of posteriors \\
         \textit{Not Named} \cite{fischer2024federated} &  Variational Inference \& MC dropout & Weighted averaging of posteriors' statistics \\
        FedPPD \cite{bhatt2024federated} & MCMC sampling & Weighted averaging of posteriors' statistics \\
        FedIvon \cite{pal2024simple} & IVON & Multiplicative aggregation of posteriors \\

 \bottomrule
    \end{tabular}
    \caption{Categorization of parametric client-side Bayesian Federated Learning methods.}
    \label{tab:bfl_methods}
\end{table*}

\paragraph{Hybrid Bayesian Deep Learning (HBDL)} 
Despite its remarkable performance, deep learning does not address crucial challenges in realistic scenarios, such as reliability and uncertainty quantification. In a recent position paper \cite{papamarkou2024position}, the authors propose Bayesian Deep Learning as a solution to the ethical, privacy, and safety challenges of modern deep learning. Acknowledging ongoing challenges with Bayesian DL, such as the additional computational cost of applying Bayesian methods to large-scale deep models, the authors envision the alternative framework of HBDL to preserve the efficiency and lower complexity of deep learning while retaining the reliability of Bayesian DL. 
HBDL is also discussed in \cite{jospin2022hands} as \textit{Bayesian inference applied only to the last (or last few) layers}. The core idea is to replace some of the layers in a Bayesian deep model with deterministic ones, thereby making the model closer to its classical deep learning counterpart. This partial Bayesian formulation makes it possible to retain uncertainty quantification capabilities while reducing complexity relative to a fully Bayesian model.

\paragraph{Bayesian Federated Learning}
BFL aims to incorporate the strengths of Bayesian deep learning into the FL framework and to provide a potential solution to the challenges outlined above. In this work, we focus primarily on parametric, client-side BFL methods \cite{cao2023Bayesian}. We categorize these methods based on their strategies for global model construction. Table \ref{tab:bfl_methods} provides a summary of this discussion.

\begin{itemize}
\item \textit{Multiplicative Aggregation of Posteriors}:
FedPA \cite{al2020federated} employs Stochastic Gradient Markov Chain Monte Carlo (MCMC) for local posterior inference, aggregating client posteriors through a product of Gaussian distributions. FOLA \cite{liu2023bayesian} approximates local posteriors using Laplace approximation. A multivariate Gaussian product mechanism is used for global posterior construction, while prior distributions derived from the global posterior guide local training, thereby enabling a continual learning setting. FedEP \cite{guo2023federated} frames FL as a distributed variational inference problem, aligning the global posterior with local posteriors through multiplicative aggregation. FedIvon \cite{pal2024simple} adopts Improved Variational Online Newton (IVON) \cite{shen2024variational} to approximate local posteriors as Gaussians with diagonal covariance, efficiently updating both mean and variance with second-order information. The server aggregates these local posteriors via the weighted product of the local Gaussians. It is important to note that approximating the global posterior as the weighted product of local posteriors is equivalent to computing their geometric mean, which also corresponds to the RKLB aggregation method introduced in this work, offering a geometric interpretation of this aggregation strategy.

\item \textit{Weighted Averaging of Posteriors' Statistics}: pFedBayes \cite{zhang2022personalized} employs variational inference to incorporate uncertainty into model parameters. From a continual learning perspective, it minimizes the KL divergence between global and local posterior distributions, balancing local reconstruction error with global alignment, which the paper presents mainly as a personalization aspect. FVBA \cite{ozer2022combine} investigates aggregating variational Bayesian neural networks using five statistical aggregation schemes. In \cite{fischer2024federated}, the authors integrate Bayesian deep learning with FL, employing variational inference and Monte Carlo Dropout for inference in local models. As in \cite{ozer2022combine}, different statistical aggregations are evaluated, highlighting the importance of the aggregation method chosen for model performance. FedPPD \cite{bhatt2024federated} leverages MCMC sampling for local posterior inference and distills posterior predictive distributions into individual deep neural networks via Stochastic Gradient Langevin Dynamics. It adopts either simple averaging or a distillation-based global aggregation approach.

\item \textit{Bayesian Posterior Inference}:
FedHB \cite{kim2023fedhb} introduces a hierarchical Bayesian framework in which local model parameters are governed by a global latent variable. Variational inference is used to optimize the local and global posteriors through block-coordinate optimization.

\end{itemize}

\section{Proposed Method} \label{sec:prop_method}
In this section, we introduce our problem formulation and present the main theoretical results. We start by formalizing the key technical aspects of the client-side BFL framework.

\paragraph{Client-Side BFL}
The Bayesian view presents a different framework for FL. 
The goal is to estimate the posterior distribution of the global model's parameters, $ p(\theta^* | \mathcal{D}) $, given the posterior distributions of local models $p(\theta_k | \mathcal{D}_k)$. 

Nevertheless, exact posterior inference is usually intractable, requiring the use of approximate inference methods instead. In this work, we consider variational inference \cite{jordan1999introduction,blei2017variational} to approximate the local posteriors given a common prior distribution $p(\theta)$ and the client likelihoods $p(\mathcal{D}_k | \theta_k)$. 

For a parametric family $\mathcal{M}_{\Psi}$ parametrized by $\psi \in \Psi$, the optimization problem seeks to identify the distribution $q_{\psi} \in \mathcal{M}_{\Psi}$ that minimizes the KL divergence from the posterior distribution $p(\theta|\mathcal{D})$, i.e., 
\begin{align}
   \min_{\psi \in \Psi} \KL(q_{\psi}(\theta) \| p(\theta | \mathcal{D}))   \label{eq: primal_VI}
\end{align}
However, the minimization in \eqref{eq: primal_VI} is not directly tractable and is commonly approached through the derivation of the \textit{Negative Evidence Lower Bound} surrogate objective $\min _{\psi \in \Psi} \mathcal{L}(\psi, D)$, where
\begin{equation}    
\label{eq:elbo}
\mathcal{L}(\psi, D) = -\mathbb{E}_{q_{\psi}(\theta)}[\log p(\mathcal{D}|\theta)]+\KL(q_{\psi}(\theta) || p(\theta)).
\end{equation}

The local models are trained by minimizing \eqref{eq:elbo} to achieve their models' posterior distributions $p(\theta_k|\mathcal{D}_k),~ \forall k \in \{1, .., N\}$. The local posteriors are then aggregated in order to get the global model's posterior $p(\theta^*|\mathcal{D})$. Given this setting, we now introduce our main assumptions regarding the common prior $p(\theta)$ and the parametric family $\mathcal{M}$, which will stay valid throughout the rest of this paper.

\begin{assumption}
    For each client, we assume the prior distribution $p(\theta)$ to be a $d$-dimensional Gaussian with independent marginals, parameterized by a zero mean vector $\mathbf{0}_d$ and an identity covariance matrix $\mathbf{I}_d$.
\end{assumption}

\begin{assumption} (Mean-field Model)
    The parametric family $\mathcal{M}_{\Psi}$ is composed of $d$-dimensional Gaussian distributions with independent marginals, i.e., $q_{\psi} = \mathcal{N}(\mu, \Sigma)$, with mean $\mu \in \mathbb{R}^d$ and diagonal covariance $\Sigma = \diag(\sigma^2_1,\ldots,\sigma^2_d)$.
\label{ass:gaussian_ind}
\end{assumption}

\begin{algorithm}[t]
\caption{\texttt{BA-BFL}: Barycentric Aggregation for Bayesian  Federated Learning}
\label{alg:BABFL}
\begin{algorithmic}[1]
\item[] \textbf{Server's Input:} number of communication rounds $R$, aggregation weights $\{w_i\}_{i=1}^N$,  global distribution's initial parameters $\psi_g^0$. 
\item[] \textbf{Client's $k$ input:} number of local training epochs $T$, local training set $\mathcal{D}_k$
\item[] 
\FOR{each round $r = 1, \dots, R$}
    \STATE Sample clients' subset $\mathcal{S}_r \subset \{1, \dots, N\}$
    \STATE Communicate $\psi_g^{r-1}$ to all clients $k \in \mathcal{S}_r$
    \FOR{\textbf{client $k \in \mathcal{S}_r$}}
        \item[] \texttt{/$^*$ $T$ epochs of Gradient Descent (GD) starting at $\psi_g^{r-1}$$^*$/}
        \STATE $\psi_k^r \gets
        \textbf{GD}( \mathcal{L}(\cdot, \mathcal{D}_k), \psi_g^{r-1}). $ ~~~~ \texttt{// Eq.\ref{eq:elbo}}
    \ENDFOR
    \item[] \texttt{/$^*$Aggregation and global update$^*$/}
    \item $\psi_g^r \gets  \textbf{D-Barycenter}(\{\psi^r_k\}_{k=1}^N, \{w_k\}_{k=1}^N)$ 
    \ENDFOR
\end{algorithmic}
\end{algorithm}

\paragraph{Bayesian Aggregation as Posteriors Barycenter}
The main novelty of this work stands in the introduction of the general Barycentric Aggregation framework for BFL (BA-BFL), an aggregation method inspired by the geometric properties of the manifold to which the local posteriors $\{ p(\theta_k|\mathcal{D}_k)\}_{k = 1 \ldots N}$ belong. Given a divergence metric $D$, we propose as a global model the barycenter $p^*_D$ of the set of clients' posteriors, i.e., the distribution that minimizes the weighted divergence from a given set. The following problem formalizes this interpretation of the aggregation process. 

\begin{problem} \label{problem:barycenter} ($D$-barycenter)
Given a statistical manifold $\mathcal{M}$, a divergence function $D: \mathcal{M} \times \mathcal{M} \to [0,\infty)$, and a set of distributions $\mathcal{S} = \{p_k\}_{k = 1 \ldots N} \subseteq \mathcal{M}$ with associated normalized weights $\{w_k\}_{k = 1 \ldots N}$, i.e., $\sum_{k = 1}^N w_k = 1$, the barycenter $p^*_D$ of the set $\mathcal{S}$ is defined as:
    \begin{align}
        p^*_D = \argmin_{q \in \mathcal{M}} \sum_{k = 1}^N w_k D(p_k|| q). \label{eq:barycenter1}
    \end{align}
\end{problem}
We now study Problem \ref{problem:barycenter} under various assumptions on the distribution set $\mathcal{S}$ and the divergence metric $D$. First, we consider the general case where $D = D_{\alpha}$, namely, the divergence belongs to the family of $\alpha$-divergences for \( \alpha \in \mathbb{R} \setminus \{ 0 \} \), without any additional assumptions on $\mathcal{S}$. As shown in~\cite{cooke1991experts,koliander2022fusion}, the corresponding barycenter \( p^*_{D_{\alpha}} \) takes the following form:
\begin{align}
    p^*_{D_{\alpha}} = \frac{\left(\sum_{k = 1}^N w_k p_k^{\alpha}\right)^{\frac{1}{\alpha}}}{\int \left(\sum_{k = 1}^N w_k p_k^{\alpha} \right)^{\frac{1}{\alpha}} d\nu}. \label{eq:barycenter:alpha}
\end{align}

Moreover, by taking the limit \( \alpha \to 0 \), i.e., corresponding to the \emph{reverse Kullback–Leibler (RKL) divergence} \( D_{RKL}(p \| q) \), the barycenter \( p^*_{RKL} \) takes the form:
\begin{align}
    p^*_{RKL} = \frac{\prod_{k = 1}^N p_k^{w_k}}{\int \prod_{k = 1}^N p_k^{w_k} d\nu}. \label{eq:barycenter:KL}
\end{align}

We now focus on the case where all $p_k \in \mathcal{S}$ are $d$-dimensional Gaussian distributions, i.e., $p_k = \Gaussian(\mu_k, \Sigma_k)$, with mean $\mu_k$ and covariance matrix $\Sigma_k$. This setting derives from Assumption \ref{ass:gaussian_ind}, where we assume that the parameters of each Bayesian layer are Gaussian distributed. For the same reasons, we are also interested in the cases where the resulting barycenter is itself Gaussian, to enforce that global and local models belong to the same family of distributions. Alas, this is not the case for the majority of $\alpha$-divergences, as discussed in the following remark.

\begin{remark} (On the $\alpha$-barycenter of a set of Gaussians) Given $\mathcal{S} = \{ \Gaussian(\mu_k,\Sigma_k) \}_{k = 1\ldots N}$, the barycenter distribution $p^*_{D_\alpha}$ in \eqref{eq:barycenter:alpha} is not Gaussian. In fact, $(p^*_{D_\alpha})^{\alpha} \propto \sum_{k = 1}^N w_k p_k^{\alpha}$, showing that the resulting barycenter is related to the Gaussian mixture obtained from the weighted sum of the elements of $\mathcal{S}$. On the other hand, $p^*_{RKL}$ is still Gaussian since the Gaussian family is closed under the product operation and \eqref{eq:barycenter:KL} is the normalized product of unnormalized Gaussians.    
\end{remark}

In light of the above technical remark, among the considered $\alpha$-divergences we focus exclusively on the case of $\alpha \to 0$, i.e., $p_{RKL}^*$, as the barycenter naturally belongs to the Gaussian family, leaving the study of other $\alpha$-divergences as future work. Given $\mathcal{S} = \{ \Gaussian(\mu_k,\Sigma_k) \}_{k = 1\ldots N}$, the RKL barycenter is $p_{RKL}^* = \Gaussian(\mu_{RKL}, \Sigma_{RKL})$, where
\begin{align}
\Sigma_{RKL} = \left(\sum_{k = 1}^N w_k \Sigma_k^{-1} \right)^{-1}, ~
\mu_{RKL} = \Sigma_{RKL} \sum_{k = 1}^N w_k \Sigma_k^{-1} \mu_k.  \label{eq:barycenter:rKL}
\end{align}  
This result is well-established in the literature and has been derived using various approaches (e.g., see \cite{Battistelli:14}).

Similarly to the $D_{RKL}$ divergence, the barycenter of a set of Gaussians in the Wasserstein-2 distance belongs to the same family. In the general setting, the parameters of the barycenter are obtained using a set of fixed-point equations \cite{Agueh:11}. However, when the set of covariance matrices $\{\Sigma_k\}_{k=1 \ldots N}$ consists of diagonal matrices, i.e., $\Sigma_k = \diag(\sigma^2_{k,1},\ldots,\sigma^2_{k,d})$, analytic expressions for the barycenter statistics can be derived, as shown in \cite{Agueh:11}:
    \begin{align}
        \Sigma_{\W2} = \left(\sum_{k = 1}^N w_k \Sigma_k^{\frac{1}{2}} \right)^2 ,\qquad
        \mu_{\W2} = \sum_{k = 1}^N w_k \mu_k. \label{eq:barycenter:W2}
    \end{align}

In the sequel, we refer to the aggregation methods resulting from \eqref{eq:barycenter:rKL} and \eqref{eq:barycenter:W2} with the acronyms RKLB and WB, respectively. We discuss the applicability of the proposed methods in HBDL, focusing on cases where part of the architecture is deterministic. In such setting, the posterior distribution $p(\theta_{k,i}|\mathcal{D})$ for the $i^{th}$ layer of the $k^{th}$ client is constrained to be a point-mass located at $\mu_{k,i}$, i.e., $p(\theta_{k,i}|\mathcal{D}) = \delta_{(\theta_{k,i} = \mu_{k,i})}$ where $\delta_{x}$ is the Dirac distribution. We investigate the behavior of the proposed methods considering the posterior $p(\theta_{k,i}|\mathcal{D}) = \Gaussian(\mu_{k,i}, \epsilon)$ in the limit case of $\epsilon \to 0$. Both \eqref{eq:barycenter:rKL} and \eqref{eq:barycenter:W2} can be shown to be well-defined in the limit, resulting in the barycenter distribution $p^*(\theta_i | \mathcal{D}) = \delta_{\left(\theta_i = \sum_{k = 1}^N w_k \mu_{k,i} \right)}$. Notably, this coincides with the arithmetic mean aggregation commonly used in deterministic FL, creating a seamless connection between deterministic and probabilistic aggregation approaches within our framework.

Compared to other state-of-the-art methodologies in parametric client-side BFL - the primary focus of this work - our barycentric perspective extends the widely used weighted multiplication of posteriors, a predominant aggregation method in the literature. We demonstrate that this approach coincides with the RKLB, thereby reinforcing its theoretical foundation. 
The other baseline used in our comparative studies in Section \ref{sec:experiments}, is the arithmetic mean of the local posteriors' statistics. In addition, we consider other statistical aggregation methods that, while explored in some comparative studies, have yet to be adopted in practical applications. More broadly, on the server side, BFL can leverage Bayesian ensembles \cite{chen2020fedbe}, also referred to as Bayesian Model Averaging (BMA), which combines predictions from sampled models to produce a more robust global estimate. Notably, BMA can also be interpreted through the lens of KL barycenters, as highlighted in [Proposition 1.5, \cite{backhoff2022bayesian}]. This barycentric perspective not only provides a strong theoretical grounding for existing methods, as illustrated in Figure \ref{fig:aggregation_mapping}, but also opens the door to exploring alternative divergence measures to enhance the aggregation process, including the WB.

To conclude this section, we provide theoretical guarantees for the convergence of BA-BFL under both WB or RKLB aggregation, as stated in the following theorem.
\begin{theorem} (Convergence) \label{th:convergence}
Under Assumption \ref{ass:gaussian_ind}, and using either RKLB or WB aggregation, BA-BFL inherits and preserves the convergence properties of FedAvg, as shown in \cite{karimireddy2020scaffold}, for non-convex scenarios with both i.i.d. and non-i.i.d. data. 
\end{theorem}

\begin{proof}
The proof of convergence of BA-BFL is based on existing results on the convergence proof of FedAvg in the non-i.i.d. setup. This connection is possible by recognizing that BA-BFL can be seen as an instance of FedAvg on the parameter space of the chosen parametric family of distributions, subject to a bijective transformation.\\
Considering a parametric family $\mathcal{M}_{\Psi}$ parametrized by $\psi \in \Psi$, let $F: \Psi \to \hat{\Psi}$ be an invertible mapping. Then, at the $i^{th}$ user, there is no difference between optimizing the local objective on $\Psi$ or a modified version optimized on $\hat{\Psi}$ via the inverse of $F$, i.e.,
\begin{align}  
    \min_{\psi \in \Psi} f_i(\psi) &= \min_{\hat{\psi} \in \hat{\Psi}} f_i(F^{-1}(\hat{\psi})) \label{th:conv:eq1} \\
    \quad \text{where} \quad f_i(\psi) &= -\mathbb{E}_{q_{\psi}(\theta)}[\log p(\mathcal{D}|\theta)]+\KL(q_{\psi}(\theta) || p(\theta)). \nonumber 
\end{align}
We are interested in the case where there exists $F$ such that the barycentric aggregation on $\Psi$ induced by a divergence $D$ is equivalent to an arithmetic mean on $\hat{\Psi}$, i.e.,
\begin{align}
    \psi_g = BA(\{\psi\}_{i = 1}^{N}) = F^{-1} \left( \sum_{i=1}^N w_i F(\psi_i) \right). \label{th:conv:eq2}
\end{align}
Under the condition that such mapping exists, then the optimization dynamics of BA-BFL on $\Psi$ is equivalent to FedAvg on $\hat{\Psi}$, i.e., \textit{BA-BFL inherits the same convergence properties of FedAvg}\footnote{For a detailed convergence proof of FedAvg in the non-i.i.d. scenario, we refer the reader to \cite{karimireddy2020scaffold}}.

Considering the family of $d$-dimensional mean-field Gaussian distributions, i.e., with diagonal covariance matrix $\Sigma = \diag(\{\sigma^2_k\}_{k = 1}^{d})$, it can be parametrized by $\psi = [(\mu_1, \sigma^2_1),\ldots, (\mu_d, \sigma^2_d)]$ with $\mu_k \in \mathbb{R}$ and $\sigma_k^2 \in \mathbb{R}^+$. Then, we can define the set of entry-wise invertible functions
\begin{align*}
    F_{RKL} (\mu_k, \sigma^2_k) = \left( \frac{\mu_k}{\sigma^2_k}, \frac{1}{\sigma_k^2} \right),  &~~F^{-1}_{RKL}(\nu_k, \psi_k) = F_{RKL}(\nu_k,\psi_k) \\
    F_{\W2} (\mu_k, \sigma^2_k) = \left( \mu_k, \sqrt{\sigma^2_k} \right),  &~~ F^{-1}_{\W2}(\nu_k, \psi_k) = (\nu_k,\psi^2_k)
\end{align*}
which satisfy \eqref{th:conv:eq2} respectively for RKL and $\W2$ divergences. Therefore, given the previous result, BA-BFL under either RKL or $\W2$ divergences enjoys the same convergence properties of FedAvg, thus concluding the proof.    
\end{proof}

\section{Experiments}
\label{sec:experiments}

We devote this section to the experimental investigation of the proposed BA-BFL. 
To this end, we conduct experimental studies on the FashionMNIST \cite{xiao2017fashion}, SVHN \cite{netzer2011reading} and CIFAR-10 \cite{Krizhevsky09learningmultiple} datasets, within a heterogeneous client setting. The datasets used exhibit varying levels of difficulty. 

To compare the proposed methodologies, we consider the following baselines:
\begin{itemize}
    \item for deterministic FL, FedAvg \cite{mcmahan2017communication} aggregates the parameters of the clients' models through arithmetic weighted average, i.e., $ \theta^* = \sum_{k=1}^N  w_k \theta_k$. 
    \item for BFL, \cite{ozer2022combine,fischer2024federated} propose different possible statistical aggregation methods detailed below:  
    \begin{itemize}
        \item Empirical Arithmetic Aggregation (EAA), 
        \[
        \mu_{EAA} = \sum_{k=1}^{N} w_k \mu_k, \quad \sigma^2_{EAA} = \sum_{k=1}^{N} w_k  \sigma^2_k.
        \]
        \item Gaussian Arithmetic Aggregation (GAA),
        \[
        \mu_{GAA} = \sum_{k=1}^{N} w_k \mu_k, \quad \sigma^2_{EAA} = \sum_{k=1}^{N} w_k^2  \sigma^2_k.
        \]  
        \item Arithmetic Aggregation with Log Variance (AALV), 
         $$
         \mu_{A A L V}=\sum_{k=1}^N w_k \mu_k, \quad
         \sigma_{A A L V}^2=e^{\sum_{k=1}^N w_k \log \sigma_k^2}.
        $$
    \end{itemize}
\end{itemize}

\paragraph{Metrics}
We evaluate the performance of the considered FL algorithms based on three criteria: accuracy, uncertainty quantification, where lower NLL indicates a better model fit, and model calibration, where lower ECE indicates better alignment between predicted probabilities and actual outcomes.

\paragraph{Experimental Setup}
To induce label shifts among the $10$ clients participating in the FL scheme, we partition the samples of each label across the clients using a Dirichlet distribution as suggested in \cite{li2020practical,yurochkin2019Bayesian,wang2020federated,wang2020tackling,lin2020ensemble,ozer2022combine}.

We assign the aggregation weights to reflect the importance of each client in proportion to the volume of data locally owned, i.e., $w_k = \tfrac{|\mathcal{D}_k|}{|\mathcal{D}|}$ where $|\cdot|$ indicates the number of samples in the dataset.

The architecture of the global and local models consists of two convolutional layers and three fully connected layers. Following an HBDL approach, we implement the last $n = 0, 1,2,3$ layers as Bayesian fully connected layers, whereas the remaining layers are deterministic. In our comparative study, increasing $n$ allows measuring the impact of additional Bayesian layers on the uncertainty quantification, model calibration, and the cost-effectiveness of the FL algorithm in time.   
\begin{table}[t]
    \centering
    \setlength{\tabcolsep}{1mm}
    \fontsize{9}{11}\selectfont
    \caption{Accuracy of the global models resulting from FedAvg, BA-BFL with RKL and $\W2$ barycentric aggregation (RKLB, WB), and BFL baseline aggregation methods (AALV, EAA, GAA). The methods are grouped based on the number of Bayesian layers (Nbl) used in the model architecture. The evaluation is conducted across three datasets (FashionMNIST, CIFAR-10, and SVHN). Bold values represent the highest performance for each dataset, while underlined values  denote the best result for the specific dataset within each group of rows corresponding to a specific Nbl.}
    \begin{tabular}{@{}c|c|c|c|c@{}}
        \toprule
        \textbf{Nbl} & \textbf{Algorithm} & \textbf{FashionMNIST} & \textbf{SVHN} & \textbf{CIFAR-10} \\ 
        \midrule
        \multirow{ 1}{*}{\textbf{0}} & \textbf{FedAvg} & 87.88 $\pm$ 0.79 & 86.06 $\pm$ 0.45 & 61.63 $\pm$ 3.11 \\ 
        \midrule
        \multirow{5}{*}{\textbf{1}} & \textbf{AALV} & 88.22 $\pm$ 0.34 & 86.52 $\pm$ 0.29 & 63.42 $\pm$ 3.02 \\ 
        & \textbf{EAA} & 88.07 $\pm$ 0.22 & 86.24 $\pm$ 0.20 & 63.69 $\pm$ 2.47 \\ 
        & \textbf{GAA} & 88.15 $\pm$ 0.31 & 86.36 $\pm$ 0.28 & 63.66 $\pm$ 2.22 \\ 
        & \textbf{RKLB} & 88.07 $\pm$ 0.36 & 86.26 $\pm$ 0.26 & 63.37 $\pm$ 2.62 \\ 
        & \textbf{WB} & \underline{\textbf{88.34 $\pm$ 0.30}} & \underline{\textbf{86.55 $\pm$ 0.37}} & \underline{63.91 $\pm$ 2.64} \\ 
        \midrule
        \multirow{5}{*}{\textbf{2}} & \textbf{AALV} & 87.62 $\pm$ 0.45 & 85.46 $\pm$ 0.10 & 65.03 $\pm$ 2.92 \\ 
        & \textbf{EAA} & 87.53 $\pm$ 0.57 & \underline{85.64 $\pm$ 0.33} & 64.02 $\pm$ 1.99 \\ 
        & \textbf{GAA} & \underline{87.82 $\pm$ 0.64} & 85.54 $\pm$ 0.44 & 64.59 $\pm$ 3.51 \\ 
        & \textbf{RKLB} & 87.59 $\pm$ 0.57 & 85.57 $\pm$ 0.45 & \underline{\textbf{65.20 $\pm$ 3.99}} \\ 
        & \textbf{WB} & 87.69 $\pm$ 0.74 & 85.57 $\pm$ 0.51 & 64.74 $\pm$ 3.29 \\ 
        \midrule
        \multirow{5}{*}{\textbf{3}} & \textbf{AALV} & \underline{88.07 $\pm$ 0.58} & 86.15 $\pm$ 0.80 & 63.71 $\pm$ 3.63 \\ 
        & \textbf{EAA} & 87.81 $\pm$ 0.54 & 86.04 $\pm$ 0.62 & 64.45 $\pm$ 1.79 \\ 
        & \textbf{GAA} & 88.02 $\pm$ 0.55 & 86.27 $\pm$ 1.02 & 64.40 $\pm$ 2.30 \\ 
        & \textbf{RKLB} & 87.77 $\pm$ 0.80 & \underline{86.53 $\pm$ 1.03} & \underline{64.55 $\pm$ 2.97} \\ 
        & \textbf{WB} & 87.54 $\pm$ 0.54 & 85.99 $\pm$ 0.68 & 64.30 $\pm$ 2.55 \\ 
        \bottomrule

    \end{tabular}
  \label{tab:acc}
\end{table}

\paragraph{Overview of the Results}
Table 1 summarizes the accuracy performance of BA-BFL using RKLB and WB, alongside the considered baseline methods. All aggregation methods are evaluated under the same model architecture, with an equal number of Bayesian layers. The results show that the proposed aggregation methods achieve performance comparable to the baselines in most scenarios. Furthermore, all Bayesian methods consistently outperform FedAvg across all evaluated datasets.
Notably, the best accuracy is obtained with a single Bayesian layer for FashionMNIST and SVHN, and with two Bayesian layers for CIFAR-10. Interestingly, increasing the number of Bayesian layers does not always result in improved accuracy.

\begin{figure*}[ht]
    \centering

    \begin{minipage}[t]{0.18\linewidth}
        \centering
        \includegraphics[width=0.98\linewidth]{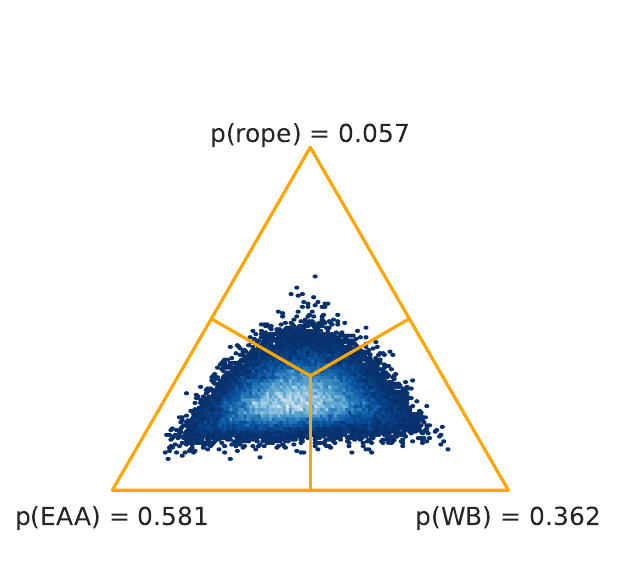}\\
        (a) EAA vs WB
    \end{minipage}
    \hfill
    \begin{minipage}[t]{0.18\linewidth}
        \centering
        \includegraphics[width=0.98\linewidth]{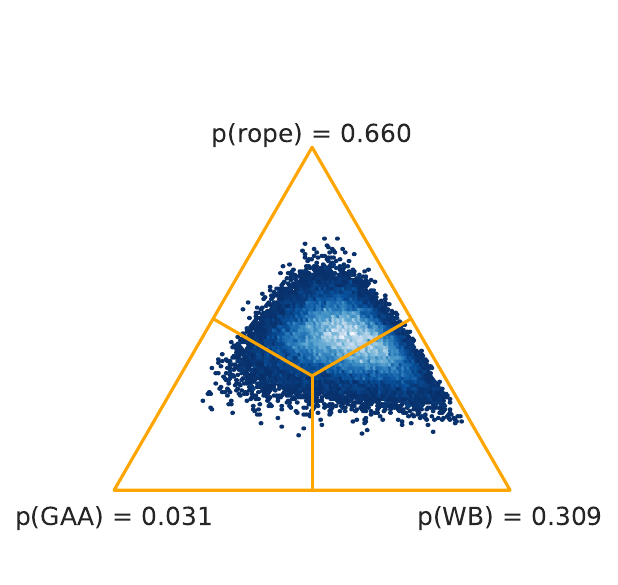}\\
        (b) GAA vs WB
    \end{minipage}
    \hfill
    \begin{minipage}[t]{0.18\linewidth}
        \centering
        \includegraphics[width=0.98\linewidth]{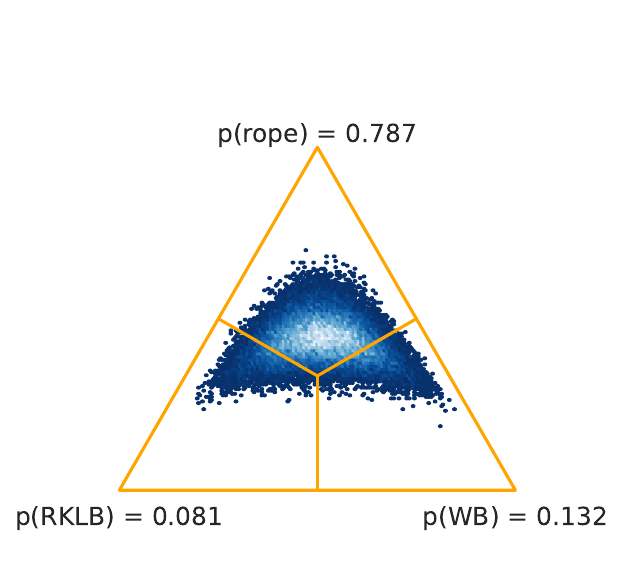}\\
        (c) RKLB vs WB
    \end{minipage}
    \hfill
    \begin{minipage}[t]{0.18\linewidth}
        \centering
        \includegraphics[width=0.98\linewidth]{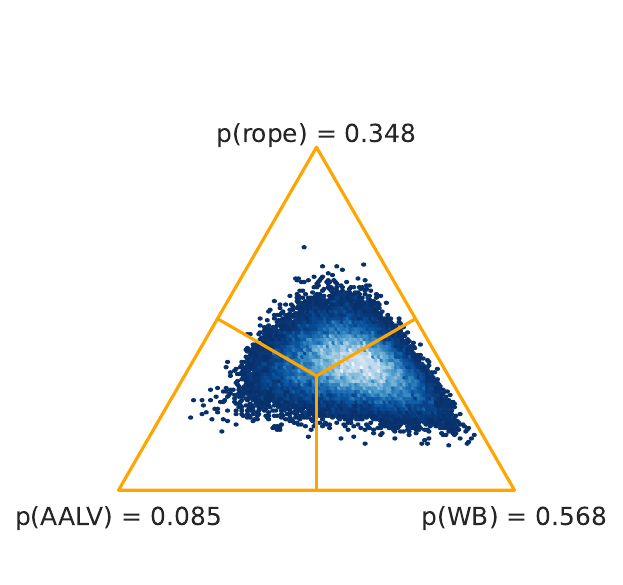}\\
        (d) AALV vs WB
    \end{minipage}
    \hfill
    \begin{minipage}[t]{0.18\linewidth}
        \centering
        \includegraphics[width=0.98\linewidth]{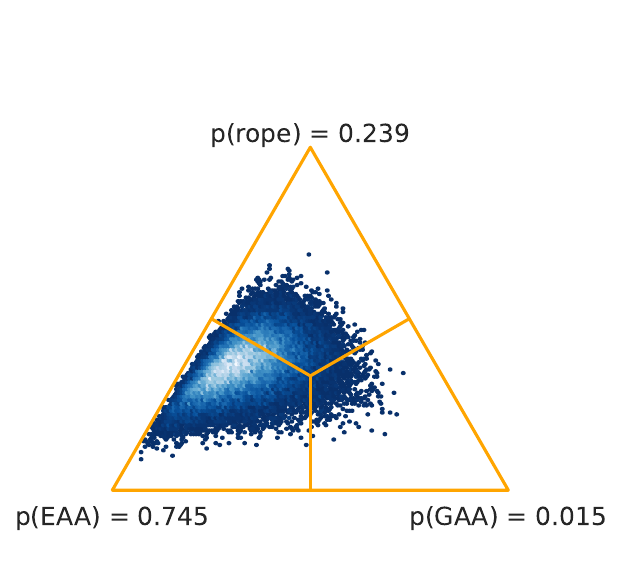}\\
        (e) EAA vs GAA
    \end{minipage}

    \begin{minipage}[t]{0.18\linewidth}
        \centering
        \includegraphics[width=0.98\linewidth]{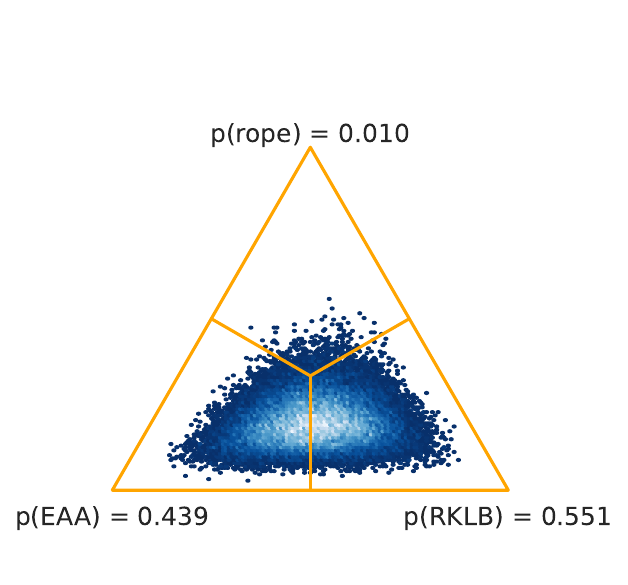}\\
        (f) EAA vs RKLB
    \end{minipage}
    \hfill
    \begin{minipage}[t]{0.18\linewidth}
        \centering
        \includegraphics[width=0.98\linewidth]{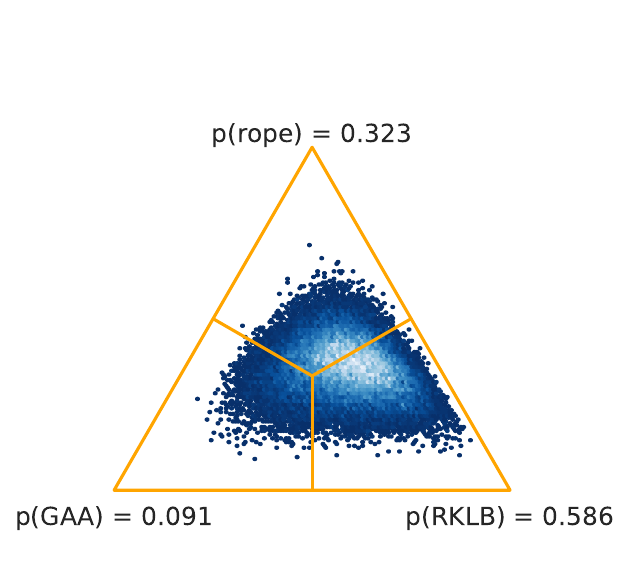}\\
        (g) GAA vs RKLB
    \end{minipage}
    \hfill
    \begin{minipage}[t]{0.18\linewidth}
        \centering
        \includegraphics[width=0.98\linewidth]{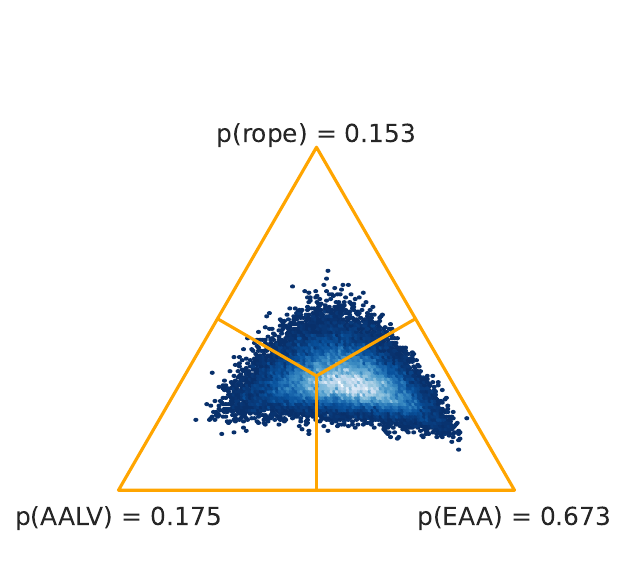}\\
        (h) AALV vs EAA
    \end{minipage}
    \hfill
    \begin{minipage}[t]{0.18\linewidth}
        \centering
        \includegraphics[width=0.98\linewidth]{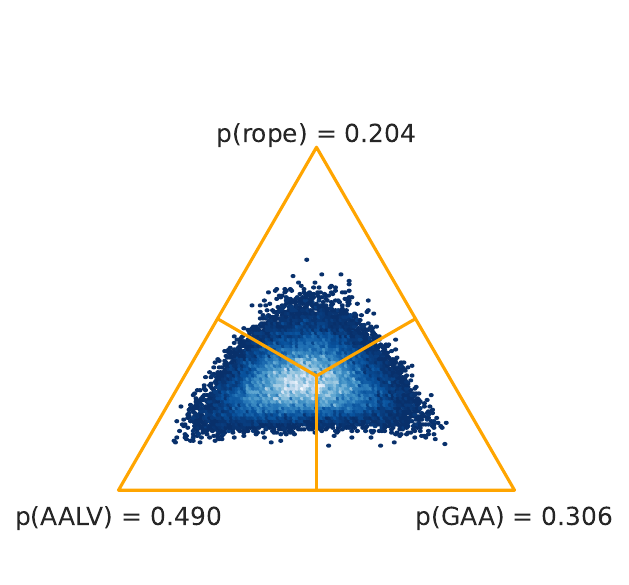}\\
        (i) AALV vs GAA
    \end{minipage}
    \hfill
    \begin{minipage}[t]{0.18\linewidth}
        \centering
        \includegraphics[width=0.98\linewidth]{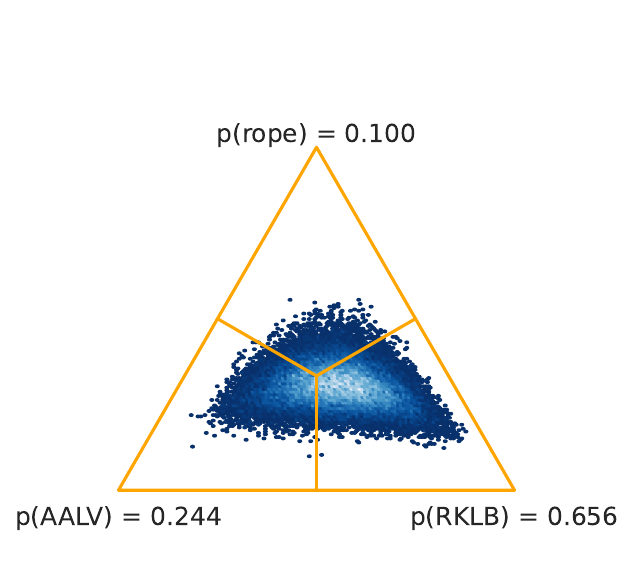}\\
        (j) AALV vs RKLB
    \end{minipage}

    \caption{Bayesian signed-rank test triangle plots comparing aggregation methods on the \textbf{NLL} metric. Each subfigure shows the posterior distribution over the relative performance between two methods.}
    \label{fig:bayes-tests-nll}
\end{figure*}

To assess whether the performance differences between aggregation methods are statistically significant, we employ the Bayesian signed-rank test as described in \cite{carrasco2017rnpbst}. For each pair of aggregation methods, we compare their performance across all shared evaluation points, i.e., for each combination of dataset, number of Bayesian layers, and random seed, ensuring a paired analysis under identical experimental conditions. The Bayesian test computes the posterior probabilities that one method outperforms, underperforms, or performs similarly to the other, where similarity is defined within a Region of Practical Equivalence (ROPE). Rather than relying on a fixed ROPE, we adopt a data-driven approach, i.e., for each comparison, the ROPE is set to the 25th percentile of the absolute differences between the two methods. This threshold serves as a conservative estimate of what constitutes a practically negligible performance. When methods behave similarly, most observed differences fall within this region, otherwise, clear performance gaps emerge beyond it. Figure~\ref{fig:bayes-tests-nll} presents the results based on the NLL metric, which is particularly informative as it captures both predictive accuracy and the quality of uncertainty estimation. As shown, the comparisons reveal no statistically significant advantage for any method, indicating broadly comparable performance across all evaluated scenarios.

\begin{figure}[ht]
    \centering
    \includegraphics[width=0.9\linewidth]{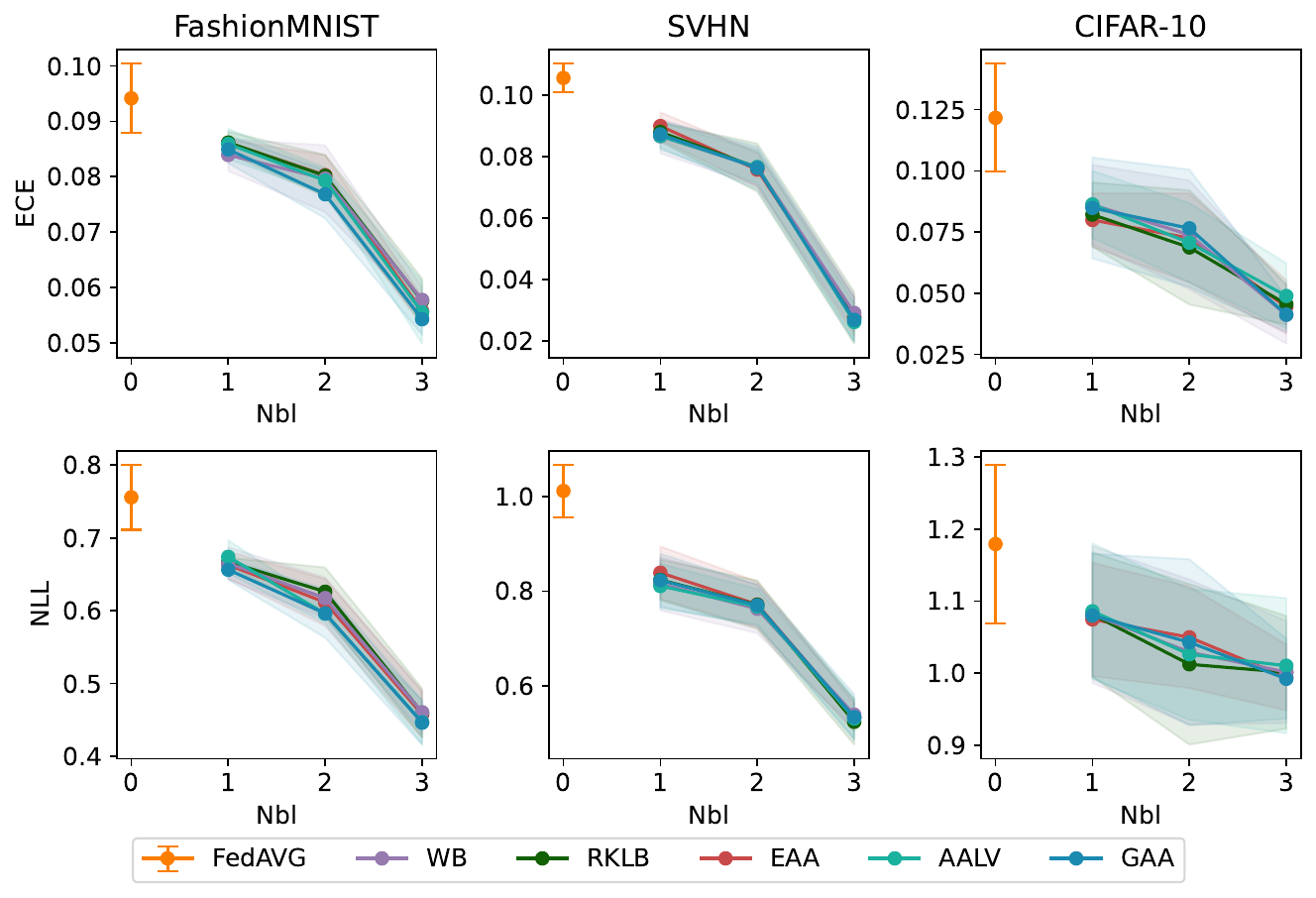}
    \caption{Effect of Bayesian Layers on Uncertainty Quantification and Model Calibration.}
    \label{fig:uq_vs_nbl}
\end{figure}

Focusing on the reliability of BFL, we observe that, regardless of the aggregation technique used, the trends reported in Fig. \ref{fig:uq_vs_nbl} indicate that \emph{increasing the number of Bayesian layers in local models improves both global model calibration and uncertainty quantification while reducing ECE and NLL scores.} However, this added Bayesian complexity often comes at the expense of reduced time efficiency. As the number of Bayesian layers increases, computational demand rises, leading to longer processing times per communication round. This trade-off between improved model reliability and increased computational cost must be carefully considered in practical applications.

Finally, in Table \ref{tab:sota}, we compare the performances of BA-BFL against two state-of-the-art parametric client-side BFL methods, pFedBayes \cite{zhang2022personalized} and pFedVem \cite{zhu2023confidence}. For the proposed method, we consider the setting with RKLB aggregation and three Bayesian layers, as it provides the best performance tradeoff. On the other hand, for pFedBayes and pFedVem, we consider the configurations detailed in their respective original papers. The results show comparable performance across all three methodologies. Our proposed approach achieves relatively better accuracy and NLL on the SVHN and CIFAR-10 datasets. In contrast, pFedVem and pFedBayes yield the best accuracy and NLL, respectively, on FashionMNIST.

\begin{table}[h]

\begin{subtable}[t]{\linewidth}
    \centering
    \begin{tabular}{lccc}
    \toprule
    \textbf{Method} & \textbf{FashionMNIST} & \textbf{SVHN} & \textbf{CIFAR-10} \\
    \midrule
    pFedVem  & \textbf{89.50 $\pm$ 0.23} & 86.32 $\pm$ 0.22 & 60.88 $\pm$ 1.44 \\
    pFedBayes  & 88.02 $\pm$ 0.39 & 86.03 $\pm$ 0.41 & 63.86  $\pm$ 1.58 \\
    Ours  & 87.77 $\pms$ 0.80 & \textbf{86.53 $\pms$ 1.03} & \textbf{64.55 $\pms$ 2.97}
    \end{tabular}
    \caption{Comparison in Accuracy}
    \label{tab:sota_acc}
\end{subtable}

\vspace{0.5cm}

\begin{subtable}[t]{\linewidth}
    \centering
    \begin{tabular}{lccc}
    \toprule
    \textbf{Method} & \textbf{FashionMNIST} & \textbf{SVHN} & \textbf{CIFAR-10} \\
    \midrule
    pFedVem  & 0.45 $\pm$ 0.02 & 0.80 $\pm$ 0.03 & 2.44 $\pm$ 0.10 \\
    pFedBayes  & \textbf{0.34 $\pm$ 0.02} & 0.66 $\pm$ 0.04 & 1.25 $\pm$ 0.06 \\
    Ours  & 0.46 $\pms$ 0.03 &  \textbf{0.52 $\pms$ 0.05}  & \textbf{1.00 $\pms$ 0.08} \\ 
    \end{tabular}
    \caption{Comparison in NLL}
    \label{tab:sota_nll}
\end{subtable}
    \caption{Comparison to state-of-the-art methods}
    \label{tab:sota}
\end{table}

\section{Conclusions and Future Work}
In this paper, we introduced BA-BFL, a novel geometric interpretation of barycenters as a solution to the BFL aggregation problem. This approach provides an explainable aggregation method. The information-geometric view of the aggregation step naturally enables operations such as clustering local posteriors directly on the statistical manifold, which has potential applications in hierarchical FL. Building on this concept, we recovered two aggregation techniques based on analytical results for Gaussian barycenters using two widely used divergences: the squared Wasserstein-2 distance and the reverse KL divergence.
We demonstrated that BA-BFL retains the convergence properties of FedAvg for non-convex loss functions and performs robustly in both homogeneous and heterogeneous data scenarios. We experimentally evaluated the proposed methods in heterogeneous settings, showing improvements over state-of-the-art methods. We also examined the impact of varying the number of Bayesian layers in an HBDL context, evaluating their effects on accuracy, uncertainty quantification, model calibration, and cost-effectiveness.
For future work, we envision several extensions, including expanding the family of distributions to include non-parametric ones and exploring alternative divergence measures. We also plan to address the personalization problem within the barycentric aggregation framework for BFL.

\section*{Acknowledgments}
\small 
The work was supported by the European Research Council (ERC) under the European Union’s Horizon 2020  Research and Innovation programme (Grant agreement No.101003431), the IMT "Futur, Ruptures \& Impacts" program, and the European Commission through the Horizon Europe/JU SNS project, ROBUST-6G (Grant Agreement No. 101139068).

\bibliographystyle{IEEEtran}
\bibliography{ref_capitalized}

\end{document}